\documentclass[conference]{IEEEtran}
\IEEEoverridecommandlockouts
\usepackage{cite}
\usepackage{amsmath,amssymb,amsfonts,amsthm}
\usepackage{algorithmic}
\usepackage{graphicx}
\usepackage{textcomp}
\usepackage{xcolor}
\usepackage{cleveref}
\usepackage{bbm}
\usepackage{dsfont}
\usepackage[ruled,vlined]{algorithm2e}
\usepackage[numbers]{natbib}
\usepackage{wrapfig}

\allowdisplaybreaks[4]

\newtheorem{theorem}{Theorem}

\newcommand{\ra}{\rightarrow}
\newcommand{\ignore}[1]{}

\newcommand{\Exp}[1]{\mathbb{E}\left[#1\right]} 

\newcommand{\floor}[1]{\left\lfloor #1 \right\rfloor}
\newcommand{\N}{\mathbb{N}}
\newcommand{\Z}{\mathbb{Z}}

\DeclareMathOperator*{\argmax}{arg\,max}

\newcommand{\M}{\mathcal{M}}
\renewcommand{\i}{^{(i)}}
\newcommand{\B}{\mathcal{B}} 
\newcommand{\s}{\mathcal{S}} 
\newcommand{\A}{\mathcal{A}}
\newcommand{\R}{\mathcal{R}} 
\newcommand{\p}{\mathcal{P}}  

\newboolean{showcomments}
\setboolean{showcomments}{true}
\newcommand{\jk}[1]{ \ifthenelse{\boolean{showcomments}}
	{\textcolor{red}{(JK says: #1)}} {} }
\newcommand{\vd}[1]{\ifthenelse{\boolean{showcomments}} 
	{\textcolor{blue}{(VD says: #1)}} {} }

\def\BibTeX{{\rm B\kern-.05em{\sc i\kern-.025em b}\kern-.08em
    T\kern-.1667em\lower.7ex\hbox{E}\kern-.125emX}}
\begin{document}

\title{Optimal Cycling of a Heterogenous Battery Bank via
  Reinforcement Learning \thanks{This work was supported by grant
    DST/CERI/MI/SG/2017/077 under the Mission Innovation program on
    Smart Grids by the Department of Science and Technology, India.}
}

\author{\IEEEauthorblockN{Vivek Deulkar}
\IEEEauthorblockA{\textit{Dept. of Electrical Engineering} \\
  \textit{IIT Bombay}}
\and
\IEEEauthorblockN{Jayakrishnan Nair}
\IEEEauthorblockA{\textit{Dept. of Electrical Engineering} \\
\textit{IIT Bombay}}
}

\maketitle

\begin{abstract}
  We consider the problem of optimal charging/discharging of a bank of
heterogenous battery units, driven by stochastic electricity
generation and demand processes. The batteries in the battery bank may
differ with respect to their capacities, ramp constraints, losses, as
well as cycling costs. The goal is to minimize the degradation costs
associated with battery cycling in the long run; this is posed
formally as a Markov decision process. We propose a linear function
approximation based $Q$-learning algorithm for learning the optimal
solution, using a specially designed class of kernel functions that
approximate the structure of the value functions associated with the
MDP. The proposed algorithm is validated via an extensive case study.

\end{abstract}


\section{Introduction}
\label{sec:intro}


Grid-scale battery storage will play a key role in reliable power grid
operation, as we transition to higher and higher penetrations of
renewable generation. While the capital cost of battery storage is
still prohibitive, costs have declined sharply over the past few years
\cite{Cole2019}. As a result, several electric utilities have
installed lithium-ion based battery storage systems with capacities
running into hundreds of megawatt-hours in recent years. 
Going forward, we should expect that multiple
battery units, of varying chemistries and capacities, would be
simultaneously connected to the grid, necessitating a sophisticated
battery management system.


In this paper, we focus on the optimal operation of a bank of
heterogenous batteries. The batteries may differ in terms of storage
capacity, ramp constraints, cycling costs, as well as losses. In
practice, such heterogenous battery banks can arise either due to the
ongoing evolution of the state of the art, and also due to
heterogeneity in use cases. Indeed, the battery chemistry (and scale)
best suited to support frequency regulation services, which require
high-frequency cycling, might be different from what is best suited to
the durnal cycling needed to match solar generation with household
electricity consumption.


We pose the optimal battery bank management problem, from the
standpoint of an electric utility, as a Markov decision process
(MDP). The state evolution dynamics of this MDP capture both the
(uncontrollable) supply-side and demand-side uncertainties, as well as
the (controllable) dynamics of the state of charge of various battery
units. The learning task is to optimally charge/discharge the battery
bank with the goal of minimizing the overall degradation of the
battery bank due to cycling. Since it is natural to assume that the
dynamics of this MDP are not a priori known precisely, we adopt a
reinforcement learning (RL) based approach. Moreover, in order to
tackle the state-action space explosion inherent in the MDP, we resort
to function approximation aided $Q$-learning (see
\cite{sutton2018reinforcement}). This involves the novel design of a
compact collection of features (a.k.a. kernel functions) that capture
the shapes of our value functions. Finally, we validate the proposed
approach via extensive case studies.

The contributions of this paper may be summarized as follows.
\begin{itemize}
\item We pose the optimal operation of a heterogenous battery bank to
  match a stochastic electricity generation process with a stochastic
  electrcity demand process as an MDP. The objective here is to
  minimize the cycling degradation of the battery bank over time.
\item In certain idealized cases, we show that a myopic greedy policy
  is optimal for the above MDP. However, this greedy policy is not
  optimal once losses and ramp constraints are taken into account.
\item We design a compact novel collection of features to aid the
  learning of the $Q$-function associated with the MDP, via linear
  function approximation. The family of features is chosen so as to
  approximate the specific structure of the value functions
  corresponding to the MDP. We then learn the feature weights via a
  stochastic approximation algorithm.
\item The proposed algorithm is validated in a case study, where it is
  shown to outperform greedy battery operation (in the presence of
  ramp constraints) as well as a naive proportional allocation policy.
\end{itemize}

\section{Model and Preliminaries}
\label{sec:model}

In this section, we describe our model for the management of a
heterogenous battery bank, in the form of a Markov decision process
(MDP). We follow the convention of using capital letters to denote
random quantities, and the correponding small letters to denote
generic realizations of those random quantities. Throughout, we use
$\N$ to denote the set of natural numbers, and $\Z$ to denote the set
of integers. For $n \in \N,$ let $[n] := \{0,1,\cdots,n\}.$

It is assumed that the learning agent manages a bank of~$N$ battery
units, labelled $1,2, \cdots, N.$ The capacity of Battery~$i$ equals
$\B^{(i)} \in \N$ units. Here, the battery capacities are specified in
terms of a suitably small unit of energy (say, 1 kWh) relative to the
scale of generation and consumption under consideration;
we always measure generation, storage, and consumption of energy as an
integer multiple of this unit. A discrete time setting is considered,
with $B^{(i)}_k \in [\B^{(i)}]$ denoting the energy stored in
Battery~$i$ at time~$k.$
The random \emph{net generation} at time~$k,$ i.e., the energy
generation minus the energy demand at time~$k,$ is denoted by
$E_k \in \Z.$ Note that a positive value of $E_k$ indicates an
instantaneous surplus in generation, whereas a negative value
indicates an instantaneous deficit. The agent must in turn apportion
this net generation across the battery bank, by charging the bank in
the former scenario, and discharging the bank in the latter (subject
to capacity and ramp constraints).

\subsection{Stochastic model for net generation}

The net generation $E_k$ at any time~$k$ equals the difference between
the (random) energy generation and the (random) energy demand in that
time slot. Thus, the net generation encapsulates both supply side
uncertainty (due to renewable generation) as well as demand side
uncertainty. We model the stochastic net generation as a function of a
certain Markov process~$\{X_k\}.$ The state of this Markov process
captures all those factors that influence supply and demand, including
weather conditions, seasonal factors, time of day, generation/demand
seen in the recent past, etc. Formally, $\{X_k\}$ is modelled as an
irreducible Discrete-Time Markov Chain (DTMC) over a finite state
space $\s_e.$ The net generation at time~$k$ is in turn a function of
the state of this chain, i.e., $E_k = f(X_k),$ where $f\colon\s_e \ra
\Z.$ Note that this model allows us to capture arbitrary dependencies
between generation and demand, as well as seasonal/diurnal variations.

\subsection{MDP description} We now formally define the MDP for optimal
management of the battery bank. An MDP is a tuple $\M = (\s, \A, \p,
\R,\gamma).$ Here, $\s$ denotes the state space, and $\A = \cup_{s \in
  \s} \A_s$ denotes the action space, with $\A_s$ being the set of
feasible actions in state~$s.$ $\p$ captures the transition structure,
where $\p(s^\prime|s,a)$ is the probability of transitioning to state
$s^\prime$ when action $a$ is played in state~$s.$ $\R$ captures the
reward structure, such that $\R(s,a)$ is the reward obtained on taking
the action $a$ in state $s.$ Finally,~$\gamma \in (0,1)$ denotes the
infinite horizon discount factor.

The state of the MDP at time~$k$ is given by $S_k=(X_k,B_k),$ where
$B_k = (B^{(i)}_k,\ 1\leq i \leq N)$ is the vector of battery
occupancies at time~$k.$ Thus, the state space is given by $\s := \s_e
\times \prod_{i=1}^N [\B_i].$

The action at time~$k$ is given by the tuple~$A_k = (A^{(i)}_k,\ 1
\leq i \leq N),$ where $A^{(i)}_k \in \Z$ is the number of energy
units injected into Battery~$i$ at time~$k;$ a negative value of
$A^{(i)}_k$ corresponds to the action of discharging Battery~$i$ by
$|A^{(i)}_k|$ energy units. Specifically, the action~$A_k$ causes the
battery occupancies to evolve as follows:
\begin{equation}
  \label{eq:batt_evolution}
  B^{(i)}_{k+1} = \floor{\eta\i (B^{(i)}_k + A^{(i)}_k)} \qquad (1 \leq i \leq N)
\end{equation}
Here, the factor $\eta\i \in (0,1)$ captures the \emph{energy
dissipation loss} of Battery~$i.$\footnote{A separate
charging/discharging loss can also be easily incorporated into the
model, as in \cite{kazhamiaka2018simple}.}

In a generic state $s = (x,b),$ the components of the action
vector~$a$ are constrained as follows:
\begin{align}
  \label{eq:ramp_constraint}
  &|a\i| \leq c\i \qquad (1 \leq i \leq N)\\
  \label{eq:capacity_constraint}
  &0 \leq a\i + b\i \leq \B\i  \qquad (1 \leq i \leq N)
\end{align}
Here, $c\i \in \N$ denotes the maximum number of energy units that can
be injected/extracted from Battery~$i$ at any epoch. In other words,
\eqref{eq:ramp_constraint} specifies \emph{ramp constraints} on the
battery bank. The constraint~\eqref{eq:capacity_constraint} enforces
the \emph{boundary conditions}, i.e., the battery cannot charged
beyond its capacity, and cannot be discharged beyond its present
occupancy. Additionally, we impose the following constraint on
$\sum_{i = 1}^N a\i:$
\begin{equation}
  \label{eq:action_sum_constraint}
  \begin{array}{rl}
    \sum_{i = 1}^N a\i &= [f(x)]_{m(b)}^{M(b)},\\
    \text{where } m(b) &:= -\sum_{i=1}^N \min \{b\i, c\i \},\\
    M(b)&:=\sum_{i=1}^N \min  \big\{(B\i-b\i ),c\i \big \}.
  \end{array}
\end{equation}
Here, $[y]_{m}^M := \min(\max(y,m),M)$ denotes the projection of~$y$
on the interval $[m,M].$ Basically, the
constraint~\eqref{eq:action_sum_constraint} states that $\sum_{i =
  1}^N a\i$ matches the net generation $f(x)$ as far as possible,
subject to ramp/capacity constraints on the battery bank. Indeed, note
that $M(b)$ equals the maximum number of energy units that can be
injected into the bank, and $-m(b)$ is the maximum number of energy
units that can be drained out of the
bank. Thus,~\eqref{eq:action_sum_constraint} ensures that the battery
bank charges to the extent possible when there is a generation
surplus, and discharges as far as possible to meet a generation
deficit. To summarize, the set of feasible actions~$\A_s$ in state $s
= (x,b)$ is defined by the
constraints~\eqref{eq:ramp_constraint}--\eqref{eq:action_sum_constraint}.

Next, note that the transition structure~$\p$ of the MDP is dictated
by that of the DTMC~$\{X_k\}.$ Specifically, on taking the action~$a$
in state~$s = (x,b),$ the first component~$x$ of the state evolves as
per the transition probability matrix of the DTMC~$\{X_k\}.$ The
second component~$b,$ which captures the battery occupancies, evolves
(deterministically) as per~\eqref{eq:batt_evolution}.

Finally, we define the reward structure~$\R.$ The reward structure
captures the cost/penalty associated with battery
cycling.\footnote{Indeed, it is well known that complete
charge-discharge cycling tends to diminish battery life (see
\cite{garche1997influence}). Therefore, it is natural to avoid
battery operation near full/empty levels, and rather try to operate it
at intermediate levels as far as possible.} We apply a penalty
whenever the present action causes each battery level to be below 20\%
or above 80\% of its capacity.\footnote{These specific thresholds are
only chosen for illustration. In practice, these thresholds can be set
specific to the chemistry of each battery.}  The incured penalty is
further proportional to the amount by which battery energy level
violates these thresholds (see Figure~\ref{fig:penalty_fun}). It has a
multiplying prefactor $r^{(\cdot)}$ which is specific to the battery
under consideration.
\begin{figure}
	\centering
	\includegraphics[scale=0.30]{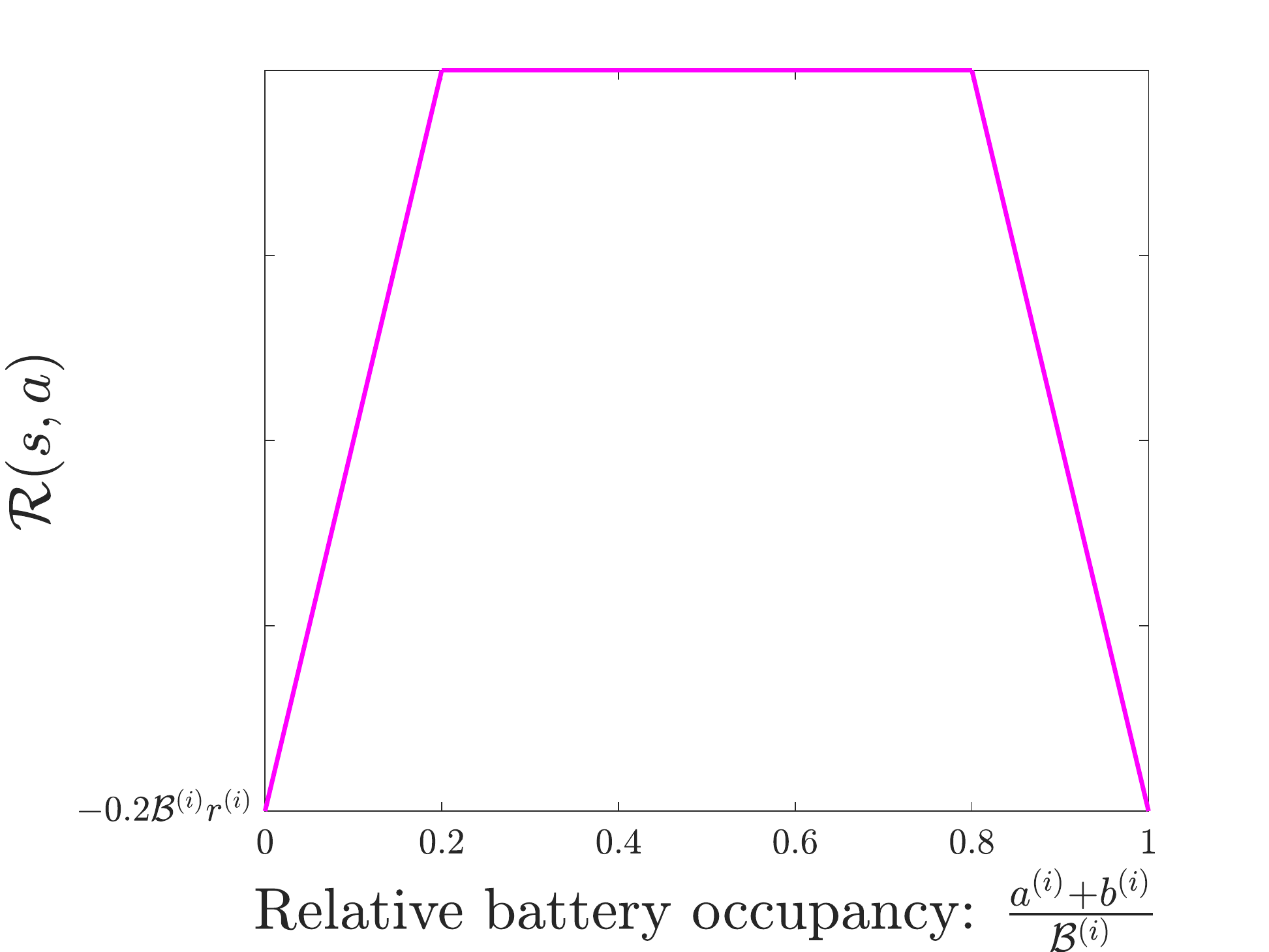}
	\caption{Reward as a function of fractional battery occupancy
          post current action}
	\label{fig:penalty_fun}
\end{figure}
Formally, the reward is given by:
\begin{equation*}
\begin{split}
  \R(s,a) = -&\sum_{i=1}^N r\i \Big\{(0.2\B\i - a\i-b\i)_+ \\
  &\quad + (a\i + b\i - 0.8\B\i)_+\Big\},
\end{split}
\end{equation*}
where $(z)_+ := \max(z,0).$ 
Note that the reward structure disincentivises
operating each battery at close to empty/full charge, the disincentive
being (potentially) heterogenous across batteries. (Batteries that
suffer a larger degradation due to cycling would be associated with a
greater penalty factor~$r^{(\cdot)}.$)


Finally, the goal of the learning agent is to maximize the infinite
horizon discounted reward, i.e.,
\begin{equation}
\label{eq:infinite_horizon_obj}
\Exp{\sum_{k =0}^{\infty} \gamma^k \R(S_k,A_k)},
\end{equation}
for any starting state~$S_0.$

\subsection{Policies and Information Structure}

In general, a policy specifies the action to be taken as a function of
the agent's observation history. A \emph{stationary deterministic
policy} $\pi\colon\s \rightarrow \A$ specifies a feasible action
purely as a function of the current state. It is well known that there
exists an optimal
stationary deterministic policy
that maximizes \eqref{eq:infinite_horizon_obj} (see
\cite{puterman2014markov}). Moreover, this policy can be computed
given the MDP primitives $(\s, \A, \p, \R,\gamma)$. In the present
context however, it is often impractical to assume prior knowledge of
the primitives. In particular, the agent may not a priori know the
transition structure~$\p$ governing the dynamics of generation/demand
evolution precisely. This motivates us to pursue a reinforcement
learning (RL) based approach to discover a near-optimal policy online.

For certain special cases (i.e., in the absence of ramp constraints
and ignoring losses), it can be proved that a simple greedy policy is
in fact optimal for the MDP, as we show in
Section~\ref{sec:greedy}. However, in general, the optimal policy is
non-trivial, and we find that the proposed RL based approach
(described in Section~\ref{sec:rl}) outperforms naive/greedy policies
(see Section~\ref{sec:case_study}).

\section{Greedy battery bank management}
\label{sec:greedy}

In this section, we define and explore a greedy policy for the MDP
posed in Section~\ref{sec:model}. This policy chooses, in any state,
an action that maximizes the instantaneous reward. Importantly, in the
context of the battery bank management problem under consideration,
this greedy policy can be applied without prior knowledge of the
transition structure~$\p.$
In other words, the greedy policy can be applied without the need for
any `learning' of the MDP dynamics. While the greedy policy is not
optimal is general, we show that it is optimal in an idealized special
case, namely when the batteries are lossless and are not subject to
any ramp constraints. From a practical standpoint, this means that the
greedy policy is expected to be near-optimal in scenarios where the
batteries are capable of fast charging/discharging, and battery losses
are negligible.

Formally, a greedy policy $\pi_{\mathrm{greedy}}$ is defined as
follows. In state~$s = (x,b),$ it chooses an action
$$\pi_{\mathrm{greedy}}(s) \in \argmax_{a \in \A_s} \R(s,a).$$
Note that if there are multiple maximizers of the instantaneous
reward, it is assumed that $\pi_{\mathrm{greedy}}$ picks one of
them.\footnote{Thus, there is really a family of greedy policies. We
use $\pi_{\mathrm{greedy}}$ here to denote any one such greedy
policy.}  Note that implementing $\pi_{\mathrm{greedy}}$ requires
knowledge of (a) $\A_s,$ which is dictated by ramp/capacity
constraints, and (b) the cycling costs, defined by $(r\i,\ 1 \leq i
\leq N)$---a reasonable expectation in practice.

The optimality of greedy battery operation in an idealized special
case is shown below.
\begin{theorem}\label{thm: greedy being optimal}
  For each Battery~$i,$ suppose that
  $c\i \geq \B\i$ (i.e., the ramp constraints are never binding), and
  $\eta\i = 1$ (i.e., the batteries are lossless). Then the
  policy~$\pi_{\mathrm{greedy}}$ is optimal for the infinite horizon
  discounted reward objective~\eqref{eq:infinite_horizon_obj}.
\end{theorem}
\begin{proof}
  The optimality of any greedy policy follows from two
  observations. First, in the absence of battery losses, at any
  time~$k,$ the total energy in the battery bank $\sum_{i = 1}^N
  B\i_k$ is conserved across \emph{all policies}. Second, since
  arbitrary energy rearrangements are possible between batteries (due
  to the absence of ramp constraints), the instantaneous reward
  attainable does not depend on the specific placement of energy
  across batteries. It then follows easily that greedy operation is
  optimal.
  \ignore{ The reason is as follows: at any time instance $k$, the sum
    total of energy across battery units is the same, both in optimal
    policy $(\pi_\star)$ and in greedy policy $(\pi_{greedy}).$
    Therefore any battery level configuration
    $(b_1^{k+1},\dots,b_N^{k+1}),$ achieved by an optimal action
    $\pi^\star (s^k),$ is always achievable with a suitable greedy
    policy action.}
\end{proof}

In general, particularly when the battery bank is heterogenous with
respect to ramp constraints and/or losses, we find that greedy
operation can be far from optimal; see
Section~\ref{sec:case_study}. Intuitively, this is because in such
cases, the total energy in the battery bank is not conserved across
different policies. Moreover, the presence of ramp constraints implies
that there are certain `preferrable' arrangements of the stored energy
across the battery bank that minimize average penalties incurred going
forward, given the random dynamics of generation/demand. The
non-triviality of the optimal policy in the presence of practical
constraints on the operation of the battery operation motivates an RL
based approach, which we address next.

\section{Reinforcement Learning based approach}
\label{sec:rl}

In this section, we describe an RL based approach for online learning
of the optimal policy for battery bank management. Specifically, we
employ a $Q$-learning algorithm with linear function approximation
\cite{sutton2018reinforcement}.
Our key contribution lies in the design of a suitable family kernel
functions that captures the specific structure of the value functions
in our formulation. The proposed approach is validated via an
extensive case study in Section~\ref{sec:case_study}.

We begin by providing some background on Q-learning. Let us denote an
optimal stationary deterministic policy of the MDP by $\pi^*$ (the
existence of such an optimal policy is guaranteed by the theory;
see~\cite{puterman2014markov}). The \emph{state value function} associated with the
MDP is defined as
$$V(s):= \Exp{\sum_{k=0}^\infty \gamma^k \R(S_k,\pi^*(A_k))\ |\ S_0 = s}.$$
The \emph{state-action value function}, a.k.a., the
\emph{$Q$-function} associated with the MDP is defined as
$$Q(s,a):= \R(s,a) + \gamma \sum_{s' \in \s} \p(s'|s,a)V(s').$$
It is well known that the
$Q$-function satisfies the Bellman equation:
$$Q(s,a):= \R(s,a) + \gamma \sum_{s' \in \s} \p(s'|s,a) \left[\max_{a'
    \in \A_{s'}} Q(s',a')\right].$$ Moreover, learning the
$Q$-function also reveals the optimal policies, since $\pi^*(s) \in
\argmax_{a \in \A_s}Q(s,a).$

$Q$-learning algorithms seek to learn the $Q$-function online,
typically via a stochastic approximation based update rule of the form
\begin{align}
  \hat{Q}(S_k, A_k) &\leftarrow \hat{Q}(S_k,A_k) +\beta_k
  \biggl[\R(S_k,A_k) \nonumber \\ &\qquad + \gamma \max_{a^\prime \in
      \A_{S_{k+1}}} \hat{Q}(S_{k+1}, a^\prime ) -
    \hat{Q}(S_k,A_k)\biggr],
  \label{eq:Q-learning}
\end{align}
where $\beta_k \in (0,1)$ is the learning rate (a.k.a., step-size) at
time $k.$ Note that the above update rule does not specify the policy
that determines the action~$A_k$ at time~$k.$
It is known that \eqref{eq:Q-learning} results in almost sure
convergence to the $Q$-function provided the chosen policy ensures
sufficient exploration of all the state-action pairs (see \cite{sutton2018reinforcement}),
and the step-size sequence satisfies $$\sum_k \beta_k =\infty, \quad
\sum_k \beta_k^2 < \infty.$$

In the specific application context under consideration, it is not
practical to implement $Q$-learning algorithms directly, given the
prohibitive number of state-action pairs. This so-called state-action
space explosion adversely impacts both the memory footprint of the
algorithm (since one must store the entire $Q$-table), and time
required to learn a near-optimal policy (since each state-action pair
needs to be visited sufficiently many times). We thus resort to
\emph{function approximation}, wherein the $Q$-function is
approximated to be a linear combination of carefully designed family
of features or kernel functions:
\begin{equation}\label{eq:fun_approx}
  Q(s,a,w) \approx \phi (s, a)^T\cdot w = \sum_{i=1}^d \phi_i(\cdot)
  w_i
\end{equation}
Here, $\phi(s,a) \in \mathbb{R}^d$ is a vector consisting of $d$
distinct features corresponding to state-action pair $(s,a)$ and $w
\in \mathbb{R}^d$ is the \emph{weight vector} to be
learned. Naturally, this is appealing when $d \ll |\{(s,a)|s\in \s, a
\in \A_s\}|.$ 

In the remainder of this section, we describe our design of the family
of features, and then state formally the proposed RL algorithm.



\subsection{Design of Kernel functions}

The $d$ distinct features chosen for each state-action pair $(s,a),$
where $s = (x,b),$ are captured in the feature vector $\phi(\cdot)$ as
follows
\begin{align*}
	\phi(s,a) = \big( \R(s,a), \  v_1 \mathds{1}_{\{x=1\}},& \ v_2 \mathds{1}_{\{x=2\}}, \nonumber\\
	& \cdots , v_{|S_e|} \mathds{1}_{\{x=|S_e|\}}   \big).
\end{align*}
Here each vector $v_j \in \mathbb{R}^{2N}$ corresponds to the $j^{th}$
state of the background DTMC $\{X_k\}$ governing the net generation,
and is defined as $v_j = $
\begin{equation}\label{eq:features_fr_a_bkgnd_state}
\Big(1, -(1-y_1)^4, -y_1^4, -(1-y_2)^4, -y_2^4, \cdots,
	-(1-y_N)^4, -y_N^4 \Big)
\end{equation}
where $y_i$ denotes the normalised
occupancy of Battery~$i,$ given by $y_i = \frac{b^{(i)} + a^{(i)}
}{\B^{(i)}}.$ Thus, in our design, $d = (2N+1)|\s|+1.$

The first feature is simply the instantaneous reward corresponding to
the state action pair $(s,a).$ 
The remaining features are interpreted as follows. For each
background state~$x,$ and for each Battery~$i,$ we introduce two
features: an increasing concave function of the normalized battery
occupancy resulting from the action~$a,$ and another descreasing
concave function of the same normalized battery occupancy; see
Figure~\ref{fig:kernel_fun}.
\begin{figure}
	\centering
	\includegraphics[scale=0.3]{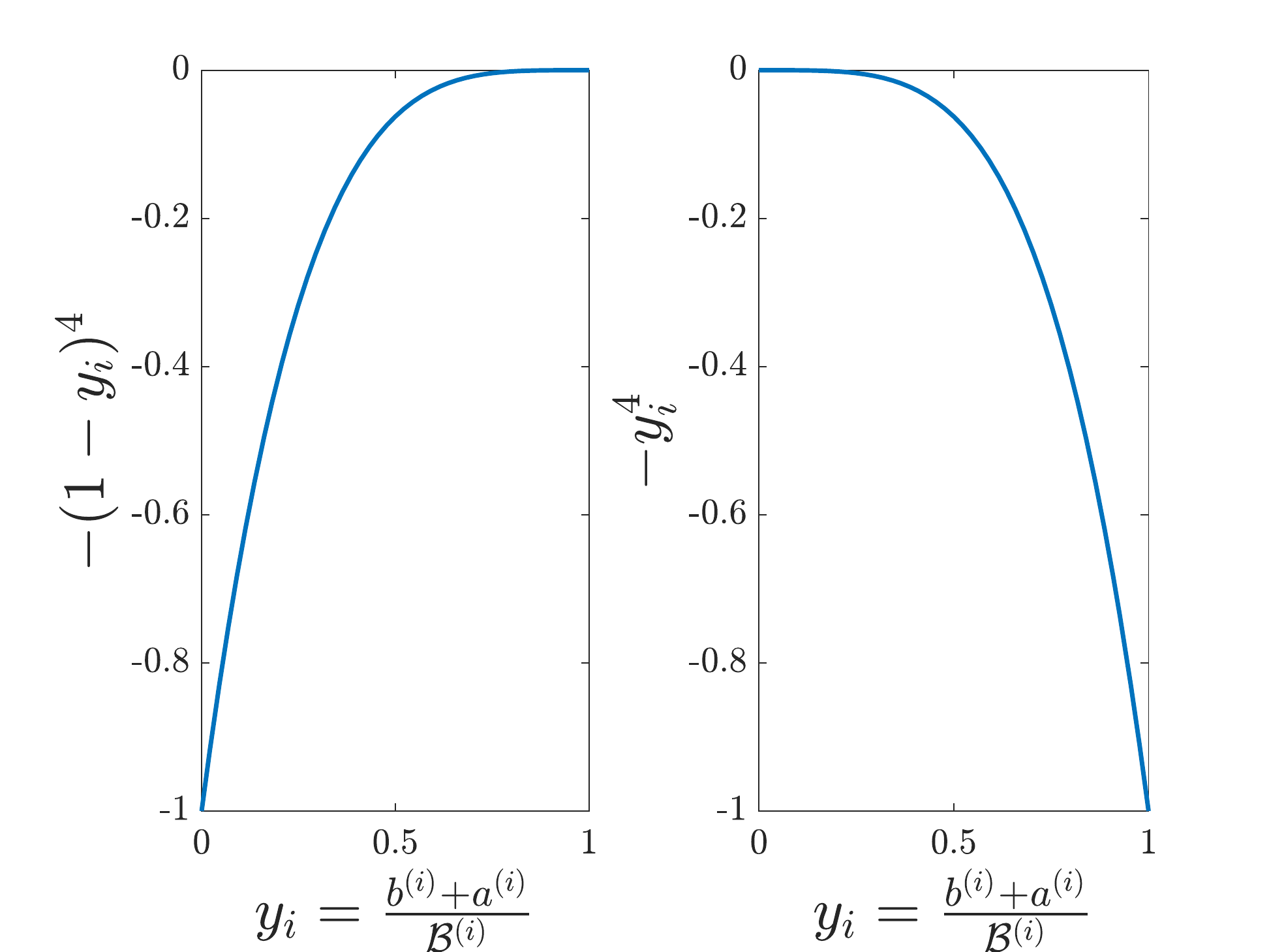}
	\caption{Two Kernel functions as a function of normalised
          battery level $y_i.$}
	\label{fig:kernel_fun}
\end{figure}
The former captures the impending (discounted) penalty arising from
Battery~$i$ becoming empty in the future; the increasing shape of this
function seeks to capture the increased likelihood of a penalty in the
near future when the present battery occupancy is lower (see the left
panel of Figure~\ref{fig:kernel_fun}). Similarly, the latter feature
captures the discounted `penalty to go' arising due to Battery~$i$
becoming full in the future; the descreasing nature of this function
capturing the increased likelihood of this event when the present
battery occupancy is closer to full (see the right panel of
Figure~\ref{fig:kernel_fun}). We introduce these two features for each
background state~$x$ separately to account for the fact that the
probability of future penalties is also dictated by~$x.$\footnote{A
more economical approach would be to cluster the background states,
and assign a feature to each cluster; this will be explored in the
future.} In addition, for each background state~$x,$ we add a constant
feature taking the value 1 (see
equation~\eqref{eq:features_fr_a_bkgnd_state}). As we show in
Section~\ref{sec:case_study}, this family of features captures the
behavior of the $Q$-function sufficiently well to enable the learning
of a near-optimal policy.


\subsection{Learning weights via stochastic semi-gradient descent}
The weight vector $w$ in \eqref{eq:fun_approx} is learned online by
stochastic semi-gradient descent method using a suffiently exploratory
policy $\pi$ over state-action pair $(s,a).$ The term `semi' refers to
the fact that the error being minimized through gradient descent is
between the \emph{estimated} values $\hat{Q}(S_k,A_k,w_k)$ and
$\left(R(S_k,A_k) + \gamma \max_{a' \in \A_{S_{k+1}} }
\hat{Q}(S_{k+1},a',w_k)\right).$ The update of weights is performed in
conjunction with an $\epsilon$-greedy policy to ensure a balance
between exploration and exploitation.

Formally, 
the weights are updated using stochastic semi-gradient descent as
follows:
\begin{align}
  w &\leftarrow w + \beta_k \big[\R(S_k,A_k) + \gamma \max_{a \in \A_{S_{k+1}}} \hat{Q} (S_{k+1},a,w) \nonumber \\
    &\quad - \hat{Q} (S_k,A_k, w) \big] \phi (S_k,A_k), \label{eq:weight_update}
\end{align}
where $\hat{Q} (s,a, w) = \phi(s,a)^T w.$ Here, $\{\beta_k\}$ is the
step-size sequence. The policy~$\pi$ that selects the action~$A_k$ at
each time~$k$ is taken to be the $\epsilon$-greedy policy,
parameterized by the sequence~$\{\epsilon_k\}.$ Formally, at time~$k,$
with probability~$\epsilon_k,$ the action~$A_k$ is picked uniformly at
randomly from the set~$\A_{S_k},$ and with
probabability~$1-\epsilon_k,$ $A_k = \argmax_{a \in \A_{S_k}}
Q(S_k,a).$

\section{Case Study}
\label{sec:case_study}

\ignore{ We did a case study with a toy model to illustrate the use of
  proposed function appoximation based RL technique for battery
  management considering the presence as well as absence of
  constraints $c_i.$ We also did a comparison between a sample
  evaluation of greedy policy, a simple \emph{naive policy} which
  apportion the net generation $E(k)$ proportional to the battery size
  $\B^{(i)},$ and the policy obtained using function approximation (an
  RL approach). In the naive policy, since the energy in measured in
  the integer unit, the apportioned net generation at battery $i$ is
  rounded off to the nearest integer. Any error in apportioning
  arising due to rounding off operation is adjusted in the $N^{th}$
  battery unit such that the resulting naive policy action $a$ should
  satisfy
  equations~\eqref{eq:ramp_constraint}-\eqref{eq:action_sum_constraint}.
}

In this section, we present simulation results on a toy model to
validate the proposed function approximation based RL algorithm. We
benchmark the performance of this policy against a certain naive
policy (which charges and discharges the battery bank in a
proportional manner) as well as a greedy policy. Our results
demonstrate that the proposed approach outperforms the greedy and
naive policies, suggesting that the proposed collection of features
does a good job of capturing the structure of the value function.

To provide a fair comparison between the three policies, we separate
the learning and the evaluation phases of our proposed
policy. Specifically, we learn the $Q$-table via function
approximation using the proposed algorithm over $T = 10^5$ time
steps. After this, we compare the `optimal' policy suggested by the
learned $Q$-function, i.e., 
\begin{equation}\label{eq:fun_approx_policy}
  \pi_{\mathrm{RL(fun\ approx)}}(s) \in \argmax_{a \in \A_s}
  \hat{Q}(s,a,w),
\end{equation}
against the benchmark policies over another run of~$T$ time steps.

While the greedy policy has been described in
Section~\ref{sec:greedy}, the naive (proportional) policy is described
as follows. In essence, it attempts to apportion the net generation
across the battery bank in a manner that is proportional to the size
of each battery, i.e., $a\i \propto \B\i.$ Of course, a perfectly
proportional allocation may not be feasble, due to interger-round off
errors, as well as ramp constraints. If this happens, the naive
policy performs a local search around the proportional allocation
until a feasible action is found; the details of this search are
omitted due to space constraints.



\subsection*{Toy Model Setup}
The state space of the background DTMC process $X_k$ is taken to be
$\s_e=\{-4,-1,1,5\},$ with the net generation associated with state~$x
\in \s_e$ being simply~$x.$ The transition probability matrix
corresponding to the background Markov process $\{ X_k \}$ is taken to
be
$$P=\begin{bmatrix}
	0	   &0.5	   &0.3	  &0.2\\
	0.5   &0       &0.1    &0.4\\
	0.3   &0.2    &0      &0.5\\
	0.3   &0.3    &0.4   &0 
\end{bmatrix}.$$ 
Two battery units are considered with penalty multipliers
$r^{(1)}=-0.1,\ r^{(2)}=-1;$ this captures the heterogeneity in
cycling costs. Various battery size configurations $(\B^{(1)},
\B^{(2)} )$ are considered for the simulation; see
Tables~\ref{table:no_constraints}-\ref{table:with_constraints}. For
simplicity, we ignore losses in these preliminary evaluations.

\ignore{The constraints, when imposed, are taken as
  $c^{(1)}=c^{(2)}=2.$ The state space $\s_e$ of the background
  process $X_k,$ along with the chosen battery configuration
  $(\B^{(1)}, \B^{(2)}),$ defines the state space $\s.$ (Recall that
  $\s := \s_e \times \prod_{i=1}^N [\B_i].$) Number of transitions
  $T,$ in the DTMC $X_k,$ will define the time horizon over which the
  function appoximation weights $w$ are learned using
  Algorithm~\ref{alg:SSGD}.  It also defines the time horizon for
  battery management over which the policy evaluation is carried out
  for greedy, optimal and RL based policy.
}

\subsection*{Simulation results}

First, we consider the case where ramp constraints are never binding
(by choosing $(c\i,\ 1 \leq i \leq N)$ large enough). In this case, it
is known that greedy operation is optimal, allowing us to evaluate
whether or not the proposed function approximation based approach is
also able to learn the optimal policy. Our results, comparing the
rewards obtained by the three policies under consideration, are
summarized in Table~\ref{table:no_constraints}. Interestingly, we see
in that all settings considered, the proposed RL based approach
achieves exactly the same reward as the (optimal) greedy policy,
whereas the naive policy performs worse. This suggests that for the
MDP instances considered here, the proposed function approximation
algorithm learns an optimal policy, validating our design of the
kernel functions.

\ignore{In the simulation, first the weights $w$ are learned using
  Algorithm \ref{alg:SSGD} where the time horizon is taken as $10^5.$
  The resultant policy obtained through learned weights (as in
  equation~\ref{eq:fun_approx_policy}), is then evaluated along with
  greedy and naive policies by taking one sample trajectory of
  background process $X_k$ with $T=10^5$ number of transitions. The
  sample path of the DTMC $X_k$ is coupled with all three policies.
  We analysed the system in the presence as well as in the absence of
  constraints $c^{(1)}$ and $c^{(2)}$ separately. The results obtained
  are as tabulated in
  Table~\ref{table:no_constraints}-\ref{table:with_constraints}.
}

Next, we consider settings where ramp constraints can be
binding. Specifically, we set $c^{(1)} = c^{(2)} = 2.$ The results for
case, corresponding to various battery size combinations, are
summarized in Table~\ref{table:with_constraints}. Note that the
proposed RL based approach outperforms both the (no longer optimal)
greedy policy, as well as the naive policy.

Validating the proposed approach on `production scale' MDP instances,
derived from real-world generation/demand traces, will be the
addressed in a forthcoming journal version of this paper.

\begin{table}[h]
	\begin{center}
		\begin{tabular}{ | c| c| c| c|}
			\hline
			$(\B^{(1)},\B^{(2)})$  &$\pi_{\text{greedy}}$ & $\pi_{\text{naive}}$   & $\pi_{\text{RL(function approx.)}}$\\		
			\hline
			2,3       &-68081 &-68081 &-68081  \\
			3,5       &-48532 &-51187  &-48532  \\
			6,10     &-63853 &-73619   &-63853  \\
			10,10    &-57859  &-74729  &-57859  \\
			15,10    &-52768 &-74842  &-52768 \\
			15,15    &-72490 &-110960  &-72490\\
			20,20   &-91273 &-144160 &-91273 \\
			\hline   
		\end{tabular}
	\end{center}
	\caption{Reward obtained across different policies when no constraints are imposed:  $S_e=\{-4,-1,1,5\},$ $T=10^5;$ cycling costs prefactors $r^{(1)}=-0.1, r^{(2)}=-1$}
	\label{table:no_constraints}
\end{table}

\begin{table}[h]
	\begin{center}
		\begin{tabular}{ | c| c| c| c|}
			\hline
			$(\B^{(1)},\B^{(2)})$  &$\pi_{\text{greedy}}$ & $\pi_{\text{naive}}$   & $\pi_{\text{RL(function approx.)}}$\\	
			\hline
			2,3       &-61212 &-61212 &-61212 \\
			3,5       &-45010 &-46861 &-44831 \\
			6,10     &-60759  &-64053   &-61774 \\
			10,10    &-66177 &-67867   &-56920 \\
			15,10    &-66146 &-67849  &-54579  \\
			15,15    &-80171 &-96938 &-68135   \\
			20,20   &-88881 &-115150 &-75473  \\
			25,25   &-99762 &-147800 &-87300  \\
			30,30   &-119320 &-184450 &-105630 \\
			40,40   &-125130 &-230460 &-113120  \\
			50,50   &-137580  &-286050 &-125350  \\
			\hline    
		\end{tabular}
	\end{center}
	\caption{Reward obtained across different policies when constraints are imposed: $c^{(1)}=c^{(2)}=2,$ $S_e=\{-4,-1,1,5\},$ $T=10^5;$ cycling costs prefactors $r^{(1)}=-0.1, r^{(2)}=-1$}
	\label{table:with_constraints}
\end{table}

\section{Related literature and concluding remarks}

In this section, we briefly survey some of the related literature, and
then conclude.

There are a few approaches in the literature for operating multiple
heterogeneous battery units. The work in~\cite{wang2015balanced}
proposes a feedback control mechanism to maintain uniform
state-of-charge (SOC) across multiple batteries without considering
the cycling cost. The work~\cite{bauer2019power} proposes another
feedback mechanism to control power flow in storage systems. A circuit
based switching mechanism to operate serially connected battery units
is studied in~\cite{shibata2001management}. A coordinated control of
multiple battery systems for frequency regulation is given in
\cite{zhu2018optimal}, which also proposes a state-of-charge (SOC)
recovery mechanism. These are non-RL approaches which deal with
battery bank operation with specific operational objectives in
mind. The supply-side uncertainty due to renewables, which is central
to the management of grid-scale storage solutions, is not considered
in these works.

On the RL front, the work in~\cite{chaoui2018deep} proposes an energy
management system using deep reinforcement learning to operate the
battery pack in an electric vechicle.

In contrast, the focus of the present paper is on the management of a
heterogenous bank of batteries connected to the grid, with the
objective of minimizing the cycling-induced degradation of the
bank. The key novelty of the approach proposed in this paper is the
use of linear function approximation to address the state-action space
explosion that occurs due to (i) the incorporation of the battery
levels into the state, and (ii) the combinatorial number of actions
that are feasible for each state. We design a novel and compact
collection of feature vectors that let us approximate the structure of
the $Q$-function, enabling the learning of a near-optimal policy
online.

Future work will focus on validating the proposed approach on a
practical utility-scale example, with cycling costs and losses modeled
using manufacturer-provided specifications. In such a case, one would
expect the net generation process to be made up of components that
fluctuate over multiple timescales, resulting in a highly non-trivial
optimal operation of the battery bank.

\bibliographystyle{IEEEtranN}
\bibliography{refs}

\end{document}